\documentclass[letterpaper, 10 pt, conference]{ieeeconf}  

\usepackage{fancyhdr}
\newcommand{\mytitle}{\textbf{Accepted final version.}
To appear in the \textit{Proceedings of the 65th IEEE Conference on Decision and Control}.\\
\copyright 2026 IEEE. Personal use of this material is permitted. Permission from IEEE must be obtained for all other uses, in any current or future media, including reprinting/republishing this material for advertising or promotional purposes, creating new collective works, for resale or
redistribution to servers or lists, or reuse of any copyrighted component of this work in other works.}
\IEEEoverridecommandlockouts                              
\usepackage{graphics} 
\usepackage{amsmath} 
\usepackage{amssymb}  
\usepackage{amsfonts}
\usepackage{mathtools}
\usepackage{bbm}
\usepackage{xcolor}
\usepackage{thmtools}
\usepackage{thm-restate}
\usepackage{siunitx}
\usepackage{multirow}
\usepackage{hyperref}
\usepackage{algorithm}
\usepackage{algpseudocode}
\usepackage{tikz}
\usepackage{wrapfig}
\usepackage{booktabs}
\usetikzlibrary{shapes.geometric, arrows.meta}

\DeclareSIUnit[quantity-product = ]\percent{\char`\%}

\hypersetup{
    colorlinks=true,
    citecolor=blue,
    linkcolor=blue,
    urlcolor=blue
}

\newtheorem{assumption}{Assumption}

\newtheorem{remark}{Remark}

\everymath=\expandafter{\the\everymath\displaystyle}

\title{
{\LARGE \bf Computationally efficient safe exploration\\in reinforcement learning}
}

\author{{Shreeram Murali, Shankar A.\ Deka, and Dominik Baumann}
\thanks{Supported by the Research Council of Finland flagship programme: the Finnish Center for Artificial Intelligence (FCAI) and the Tandem Industry Academia Seed funding from the Finnish Research Impact Foundation.}
\thanks{S.~Murali and D.~Baumann are with the Cyber-physical Systems Group; S.~A.~Deka is with the Nonlinear Systems and Control Group, Aalto University, Espoo, Finland (Email: {\tt\small firstname.lastname@aalto.fi}).}%
}

\begin{document}

\maketitle
\thispagestyle{fancy}
\pagestyle{fancy}

\begin{abstract}

Reinforcement learning in real-life applications requires safety guarantees during exploration. Typical reinforcement learning algorithms do not provide such guarantees, and many modifications that do rely on Gaussian processes (GPs), which have a large computational cost. We propose a computationally lightweight algorithm based on the Nadaraya-Watson estimator that safely explores and optimizes constrained Markov decision processes (MDPs). Our algorithm, \textsc{CoLSafe-MDP}, uses an estimator that scales in constant-time with bounds on the estimates, a significant improvement from its GP-based counterparts that scale cubically with the number of data points. We then evaluate its performance in a grid-based environment and on observational Martian terrain data.

\end{abstract}

\section{INTRODUCTION}

Suppose a planetary rover on the surface of Mars is tasked with collecting geological samples in an unfamiliar landscape. The rover must navigate terrain in which the scientific value of each location and the safety of the terrain are unknown \textit{a priori}. Furthermore, the rover operates under strict computational resource constraints: onboard processing is limited, and decisions are made in real-time. This example captures the core challenge we address in this paper: how can an agent safely explore an unknown environment and learn a near-optimal policy with limited computational resources?

Standard reinforcement learning (RL) has been widely successful in recent years, especially in structured environments such as video and board games \cite{silver_mastering_2016}. However, typical RL algorithms often disregard safety entirely in favour of trial-and-error exploration. This can be potentially catastrophic in applications such as autonomous driving, robotics, and space exploration, where people and property are involved \cite{pmlr-v80-fazel18a, koppejan_neuroevolutionary_2011}. It is therefore crucial that safety is maintained not just after convergence, but throughout exploration. 

A natural formalism for including safety constraints into RL is through constrained Markov decision processes (CMDPs) \cite{altman_constrained_2021, turchetta_safe_2016}. Here, the agent must satisfy a set of constraints while maximizing some cumulative reward. While this is a strong framework, most existing work provides theoretical guarantees only after policy convergence \cite{ding_last-iterate_2023} or ensure zero constraint violation only in expectation \cite{pmlr-v258-kalagarla25a}. Neither is sufficient when even a single safety violation can lead to damage. Ensuring safety with high probability at every step during learning is a substantially harder problem.


The safe Bayesian optimization literature addresses this problem. \textsc{SafeOpt}~\cite{sui2015SafeExplorationOptimization} and its extensions~\cite{fiedler2024SafetySafeBayesian, berkenkamp_safe_2017} use Gaussian processes to construct confidence intervals over unknown constraint functions and only evaluate inputs that are safe with high probability. However, \textsc{SafeOpt} is essentially a bandit algorithm: it can evaluate any input it certifies as safe, regardless of which inputs it evaluated before. In an MDP, the agent can only reach states through the environment dynamics, so a safe state is useful only if a safe path leads to it and away from it. Prior work~\cite{turchetta_safe_2016, wachi2018SafeExplorationOptimization} extends GP-based safe exploration to this setting by accounting for reachability and returnability, and \textsc{SNO-MDP}~\cite{wachi2020SafeReinforcementLearning} further provides near-optimality guarantees. All of these methods, however, rely on GP inference, whose posterior updates scale cubically with the number of observations~\cite{rasmussen2008Gaussianprocessesmachine}. This limits them to small state spaces or short horizons.

\subsubsection{Contributions} Therefore, in this paper, we propose \textsc{CoLSafe}-MDP, an algorithm for safe exploration and near-optimal policy learning in constrained MDPs that is both computationally lightweight and theoretically grounded. Our contributions are as follows:

\begin{itemize}
    \item We introduce a recursive formulation of the Nadaraya-Watson estimator \cite{nadaraya-1964, watson1964SmoothRegressionAnalysis} for CMDPs that exploits the finite state space and compact kernel support to achieve constant-time updates per observation, a significant improvement from the cubic scaling of GP-based methods.
    \item We derive confidence intervals for the estimator and use them to construct pessimistic and optimistic safe sets. 
    \item On the basis of these, we prove that \textsc{CoLSafe-MDP} does not violate the safety constraint with high probability and converges to a near-optimal policy in finite time. Furthermore, our theoretical guarantees require only Lipschitz continuity of the unknown functions (with a known constant), a milder assumption than~\cite{wachi2020SafeReinforcementLearning}. 
    \item We evaluate \textsc{CoLSafe-MDP} on synthetic grid worlds and on real Martian terrain data. The algorithm discovers nearly three times as many safe states as the GP-based baseline while running at a fraction of the wall-clock time, all without a single safety violation.
\end{itemize}

\subsubsection{Related work} There are many approaches to incorporating safety in RL \cite{garcia_comprehensive_2015}. Model-based approaches associate safety with state constraints to guarantee stability~\cite{berkenkamp_safe_2017, fisac_probabilistically_2018}, whereas model-free approach such as policy optimization~\cite{achiam} and Lyapunov-based methods \cite{chow_lyapunov-based_2018} guarantee constraint satisfaction only upon convergence or in expectation. Recent policy gradient methods for CMDPs have strengthened these guarantees, with feasibility guarantees throughout learning \cite{ni_safe_2025}. These methods target large or continuous state-action spaces with parametrized policies, an adjacent problem setting to the one we consider in our work. 

When the environment is unknown, GP-based confidence intervals offer a principled way to certify safety. \textsc{SafeOpt}~\cite{sui2015SafeExplorationOptimization} introduced this in the bandit setting, and \cite{turchetta_safe_2016} extended this to finite MDPs. Later, \cite{wachi2018SafeExplorationOptimization} and \cite{wachi2020SafeReinforcementLearning} unified this into \textsc{SNO-MDP}, a two-phase algorithm that first explores safely and then optimizes the reward, with provable guarantees on safety and optimality. However, these methods still suffer from the cubic computational complexity of GPs. To overcome this, more recent approaches have used nonparametric methods, such as the Nadaraya-Watson estimator \cite{nadaraya-1964, watson1964SmoothRegressionAnalysis}, to replace GPs in \textsc{SafeOpt} \cite{baumann2024computationallylightweightsafe, baumann2025SafetyOptimalityLearningBased}. In~\cite{baumann2025SafetyOptimalityLearningBased}, the authors implement a version of the estimator that scales linearly in time. However, it is possible to improve this to sublinear computational complexity with some preprocessing and the use of sophisticated data structures~\cite{rao1996PAClearningfunctions, pmlr-v300-murali26a}. Moreover, these approaches and improvements are not targeted for MDP settings. In our work, we present a version suited to CMDPs that is recursive and scales with $\mathcal{O}(1)$ in time.

\section{Problem setting}
\label{sec:problem}

We consider an infinite-horizon CMDP as the tuple
\begin{align*}
    \mathcal M = \langle \mathcal S, \mathcal A, f, r, g, \gamma \rangle, 
\end{align*}
where $\mathcal S$ is a finite set of states $\{ s \}$, $\mathcal A$ is a finite set of actions $\{ a \}$, $f : \mathcal S \times \mathcal A \to \mathcal S$ is a known state transition function, $r: \mathcal S \to \mathbb [0, R_{\text{max}}]$ is the reward function, $g : \mathcal S \to \mathbb R$ is the safety constraint function, and $\gamma \in (0, 1)$ is a discount factor. In our setting, both the reward function $r$ and constraint function $g$ are unknown \textit{a priori}. At every time step $t \in \mathbb N$, the agent must be in a `safe' state defined as $g(s_t) \ge h$, where $h \in \mathbb R$ is the safety threshold.  

 The agent's behaviour is specified by a policy $\pi : \mathcal S \to \mathcal A$, which defines the action to take in each state. The value of a policy represents the expected discounted cumulative reward from a certain state $s_t$. We denote this value function as $V^{\pi}_{\mathcal M}$. Our objective is then to find the policy $\pi$ that
\begin{align*}
    \text{maximizes:}& \qquad V^{\pi}_{\mathcal M} = \mathbb E \left[ \sum_{\tau = 0}^{\infty} \gamma^\tau r(s_{t + \tau}) \, \middle| \, s_t \right], \\
    \text{subject to:}& \qquad  g(s_{t+\tau}) \ge h, \quad \forall \tau \in [0, \infty).
\end{align*}


Our problem is centred around estimating the unknown reward and safety functions $r$ and $g$ from data. We assume to receive noisy measurements of $r$ and $g$ after every experiment with a certain policy. Thus, we introduce an assumption about the nature of noise in our measurements. For notational convenience, we define a vector $m(s_t) = [r(s_t) \quad g (s_t)]^\top$ that stacks the unknown reward and safety functions. We denote its components as $m_0(s) = r(s)$ and $m_1(s) = g(s)$.
\begin{assumption}
    \label{ass:noise}
    At every $t$, we receive measurements $\tilde{m}(s_t)$ of the reward and safety functions as $\tilde{m}(s_t) = m(s_t) + \omega_t$, where $\omega_t$ is a zero-mean noise vector whose components are each conditionally $\sigma$-sub-Gaussian.
\end{assumption}

That is, we consider $\omega_t$ to be a two-dimensional real process adapted to the natural filtration $(\mathcal F _t)_{t \ge 0}$ defined by the measurement equation, where $\mathcal F_{t-1} = \sigma(s_{1:t}, a_{1:t-1}, \tilde{m}_{1:t-1})$ represents the $\sigma$-algebra generated by all states visited up to and including time $t$, together with the actions taken and measurements received up to time $t-1$. Since $r$ and $g$ are unknown \textit{a priori}, the agent cannot certify any state as safe at $t=0$, and its first action would be unconstrained. Thus, we require the following assumption.
\begin{assumption}
    \label{ass:safe-set}
    We assume the agent starts its exploration in an initial set of states $S_0 \subseteq \mathcal S$ that is known to be safe, i.e., $g(s) \ge h$ for all $s \in S_0$.
\end{assumption}

To overcome the difficulties associated with estimating the safety of the environment based on measurements, we invoke regularity assumptions on the unknown reward and safety functions $r$ and $g$.
\begin{assumption}
    \label{ass:lipschitz}
     The functions $r$ and $g$ are Lipschitz-continuous with a known Lipschitz constant $L < \infty$.
\end{assumption}

Invoking Lipschitz continuity to provide guarantees is very common in the control and safe learning literature~\cite{Magureanu:2014, brunke2022safe, baumann2025SafetyOptimalityLearningBased}. Such an assumption is closely tied with invoking an upper bound on the norm in a reproducing kernel Hilbert space \cite{fiedler2024SafetySafeBayesian}. However, unlike~\cite{wachi2020SafeReinforcementLearning}, we do not explicitly assume a bounded RKHS norm. Although these assumptions often go hand-in-hand~\cite{tokmak2025safeexplorationreproducingkernel}, the Lipschitz constant of a function can be more intuitive to estimate based on data~\cite{Strongin1973, huangSampleComplexityLipschitz}.

\section{Safe exploration algorithm}

In this section, we introduce our proposed algorithm for safe and optimal policy learning. 

\subsection{Nadaraya-Watson estimator}

The Nadaraya-Watson estimator is a nonparametric kernel regressor that produces the conditional expectation of a random variable~\cite{nadaraya-1964, watson1964SmoothRegressionAnalysis}. We use the standard form of this estimator to estimate $m$ as
\begin{align}
    \label{eq:nwe}
    \hat{m}_t(s) &:= \sum_{\tau=1}^{t} \frac{K_{\lambda}(s, s_\tau)}{\kappa_{t}(s)} \tilde{m}(s_\tau), 
\end{align}
where
\begin{align}
    \label{eq:nwe-kernel-defs}
    \kappa_{t}(s) &:= \sum_{\tau=1}^{t} K_{\lambda}(s, s_\tau), \nonumber \\
    K_{\lambda}(s, s_\tau) &:=  K\left(\frac{\|s - s_\tau \|}{\lambda}\right). 
\end{align}

Here, $K$ is the kernel function, $\lambda$ is the bandwidth parameter, and $\tilde m(s)$ are the measurements of the reward and constraint functions. In~\eqref{eq:nwe-kernel-defs} and in the remainder of this paper, we use $\| \cdot \|$ to denote the $L^2$ norm. 

To derive guarantees on the estimate produced by~\eqref{eq:nwe}, we now introduce an assumption on the nature of the kernel $K$. 

\begin{assumption}
    \label{ass:kernel}
    There exists a constant $0 < \chi_K < \infty$ such that the kernel $K : \mathbb{R}^d \to \mathbb{R}$ satisfies $\chi_K \le K(v) \le 1$ for all $\|v\| \le 1$, and $K(v) = 0$ for all $\|v\| > 1$.
\end{assumption}

Since the kernel function is user-defined, this assumption can be satisfied either by explicitly selecting a kernel that matches these conditions, for instance the box kernel, or by formulating one in line with this assumption. 

\subsection{Confidence intervals}
\label{sec:ci}

Next, for safety and optimality guarantees, we require bounds on how close each component of our estimate $\hat{m}_i(s)$ is to the true $m_i(s)$. 

\begin{restatable}{lemma}{lemmaone}
\label{lem:bounds}
Let Assumptions~\ref{ass:noise}--\ref{ass:kernel} hold. For every $\delta \in (0, 1)$, with probability at least $1-\delta$,
\begin{align*}
    | \hat{m}_i(s) - m_i(s) | \le \beta_t(s)
\end{align*}
simultaneously for all $i \in \{0, 1\}$, all time steps $t \ge 1$, and all states $s \in \mathcal{S}$. The confidence bound width $\beta_t(s)$ is defined as
\begin{align*}
    \beta_t(s) = 
    \begin{cases} 
      +\infty, & \text{if } \kappa_t(s) = 0, \\
      L\lambda + \frac{\sigma}{\kappa_t(s)} \varepsilon_t(s) & \text{otherwise,}
   \end{cases}
\end{align*}
where 
\begin{align*}
    \varepsilon_t(s) &= \sqrt{2 \ln\left( \frac{2|\mathcal{S}|}{\delta} \sqrt{1 + \nu_t(s)} \right)} \sqrt{1 + \nu_t(s)}, \\
    \nu_t(s) &= \sum_{\tau=1}^t K_\lambda^2(s, s_\tau). 
\end{align*}
\vspace{0.1em}
\end{restatable}

The proof of this lemma builds upon \cite[Lem.~1]{baumann2025SafetyOptimalityLearningBased}. The confidence bound $\beta_t(s)$ holds uniformly over both components, all time steps, and the finite state space $\mathcal S$ via a union bound. We collect our proofs in Section~\ref{sec:proofs}. 

\subsection{Updating the estimator in constant time}
\label{sec:updating}

In its naive form, the estimator in~\eqref{eq:nwe} scales with $\mathcal O(t)$ in time since it iterates over $t$ points. Nevertheless, a specific quirk in our problem setting of constrained MDPs allows us to implement a recursive variant of the Nadaraya-Watson estimator that scales in constant-time, i.e., $\mathcal O(1)$.

In our setting, we make use of two properties to implement a recursive grid-based approach. First, since the agent operates on a finite set of states $\mathcal S$ (e.g., grids on a map), we know exactly what all the possible query points $s$ will be. Second, the kernel function drops to exactly $0$ after a certain distance controlled by the bandwidth parameter $\lambda$.

This discretized setting allows us to avoid recalculating the sum in~\eqref{eq:nwe} over all past $t$ observations. Instead, we can trade storage for speed by maintaining sufficient statistics for each state $s$ in our grid $\mathcal S$. For the estimator in~\eqref{eq:nwe}, these are $\theta_t \coloneqq \sum_{\tau = 1}^t K_\lambda (s, s_\tau) \tilde{m}(s_\tau)$, and $\kappa_t$ from~\eqref{eq:nwe-kernel-defs}. 

When a new observation $(s_{t+1}, m_{t+1})$ arrives, we only update the statistics $\theta_{t}(s)$ and $\kappa_t(s)$ for states within the compact kernel's bandwidth (i.e., where $\| s - s_{t+1} \| \le \lambda$) as
\begin{align}
    \label{eqn:constant-time-update}
    \theta_{t+1}(s) & \leftarrow \theta_{t}(s) + K_{\lambda}(s, s_{t+1}) m_{t+1}, \nonumber \\
    \kappa_{t+1} (s) & \leftarrow \kappa_t(s) + K_\lambda (s, s_{t+1}).
\end{align}

With this, we can update our estimate as
\begin{align}
    \label{eqn:constant-time-estimate}
    \hat{m}(s) = \frac{\theta_{t + 1}(s)}{\kappa_{t+1}(s)}.
\end{align}

For the confidence intervals, the only data-dependent term is $\nu_t(s)$, which can be recursively updated as 
\begin{align}
    \nu_{t+1}(s) \leftarrow \nu_t(s) + K_\lambda ^2 (s, s_{t+1}). 
\end{align}

\subsection{Pessimistic and optimistic safe sets}

With valid confidence intervals established, we now define two kinds of predicted safe spaces as inferred by our estimator: the predicted pessimistic and optimistic safe sets. These notions are similar to those in \cite{wachi2020SafeReinforcementLearning} and are inspired by \cite{turchetta_safe_2016, wachi2018SafeExplorationOptimization, turchetta_safe_2019}. The pessimistic safe set consists of states that satisfy our estimates of the safety constraint with high probability. The optimistic safe set contains any state that could potentially be safe. Furthermore, we filter this set to consider those states that are also reachable and returnable. The safety is evaluated on the basis of the bounds inferred by our estimator. 

At any step $t$, the Nadaraya-Watson estimator provides a valid confidence interval for the safety function $g(s)$ at state $s$, denoted as $Q_t(s) := [\hat{m}_{1}(s) - \beta_t(s), \hat{m}_{1}(s) + \beta_t(s)]$. To ensure that our safe set does not shrink over time, we intersect these intervals sequentially. We define $C_t(s) = Q_t(s) \cap C_{t-1}(s)$, initializing with $C_0(s) = [h, \infty)$ for all known safe states $s \in S_0$, and $C_0(s) = (-\infty, \infty)$ for all others. The tightened lower and upper bounds are denoted by $l_t(s) := \min C_t(s)$ and $u_t(s) := \max C_t(s)$. Rather than relying on computationally expensive GP posteriors, we construct these spaces using the recursive confidence bounds established in Lemma~\ref{lem:bounds}. Both safe spaces are initialized with the known safe seed from Assumption~\ref{ass:safe-set}, $\mathcal{X}_0^- = \mathcal{X}_0^+ = S_0$.

\subsubsection{Predicted pessimistic safe set}

We now define two sets for the probabilistic safety guarantee. First, we have $S_t^-$, the set of states that satisfy the safety constraint $h$ with high probability. By using the Lipschitz continuity (Assumption~\ref{ass:lipschitz}) of the safety function, we can certify states that have not yet been directly sampled but are sufficiently close to known safe states as
\begin{align*}
    S_t^- = \{s \in \mathcal{S} \mid \exists s' \in \mathcal{X}_{t-1}^- : l_t(s') - L \cdot \| s - s' \| \ge h\}.
\end{align*}
However, certifying a state as safe is insufficient if the agent cannot safely navigate to it or return from it. We enforce topological safety using reachability and returnability operators. Reachability ensures the agent can enter a newly certified state from its current safe set, while returnability ensures it can leave again. Let $R_{\text{reach}}(X)$ be the set of states reachable from a set $X$ in one step. Let $R_{\text{ret}}(X, \bar{X})$ denote the states in $X$ from which the agent can transition back to a target set $\bar{X}$ in one step. An $n$-step returnability operator is defined recursively as $R_{\text{ret}}^n(X, \bar{X}) = R_{\text{ret}}(X, R_{\text{ret}}^{n-1}(X, \bar{X}))$. The limit as $n \to \infty$, denoted $\bar{R}_{\text{ret}}(X, \bar{X})$, yields all states that can eventually return to $\bar{X}$ through $X$.

With this, the predicted pessimistic safe space $\mathcal{X}_t^-$ is thus defined as the subset of $S_t^-$ as
\begin{align}
    \label{eq:pessimistic_safeset_expansion}
    \mathcal{X}_t^- = \{s \in S_t^- \mid s \in R_{\text{reach}}(\mathcal{X}_{t-1}^-) \cap \bar{R}_{\text{ret}}(S_t^-, \mathcal{X}_{t-1}^-)\}.
\end{align}
Therefore, $\mathcal X_t^-$ is the subset of $\mathcal S_t^-$ consisting of states that are reachable from the previous safe space $\mathcal{X}_{t-1}^-$ and from which the agent can return to it. By restricting the policy to actions that transition exclusively within $\mathcal{X}_t^-$, the algorithm theoretically guarantees that the agent will not violate the safety threshold $h$.

\subsubsection{Predicted optimistic safe set}

To guide exploration, we must also identify states that have the potential to be safe. Using the upper bound of our confidence intervals, we define an optimistic safe set $S_t^+$ as
\begin{align*}
    S_t^+ = \{s \in \mathcal{S} \mid \exists s' \in \mathcal{X}_{t-1}^+ : u_t(s') - L \cdot \| s - s' \| \ge h\}.
\end{align*}
States outside of $S_t^+$ are deemed unsafe with high probability and are discarded from the search space. Thus, the predicted optimistic safe space $\mathcal{X}_t^+$ is
\begin{align}
    \mathcal{X}_t^+ = \{s \in S_t^+ \mid s \in R_{\text{reach}}(\mathcal{X}_{t-1}^+) \cap \bar{R}_{\text{ret}}(S_t^+, \mathcal{X}_{t-1}^+)\}.
\end{align}

In other words, $\mathcal{X}_t^+$ contains any state that could potentially be safe. The agent's goal during the exploration phase is to select paths within the pessimistic set $\mathcal{X}_t^-$ that gather observations to systematically convert elements of the optimistic set $\mathcal{X}_t^+$ into certified safe states. 

\subsection{The algorithm: \textsc{CoLSafe-MDP}}
\label{sec:algorithm}

We now introduce our proposed algorithm, \textsc{CoLSafe-MDP}, that efficiently achieves a near-optimal policy in finite time while guaranteeing safety. Algorithm~\ref{alg:algorithm} displays the pseudo-code overview. 


\begin{algorithm}[t]
\caption{\textsc{CoLSafe-MDP}}
\label{alg:algorithm}
\begin{algorithmic}[1]
\State \textbf{Input:} State space $\mathcal{S}$, transition function $f$, safe seed $S_0$, Lipschitz constant $L$, bandwidth $\lambda$, optimality bound $\bar{\beta}$, safety threshold $h$, desired confidence $\delta$
\State Initialize $(\theta_0, \kappa_0, \nu_0) \leftarrow 0$; $C_0(s) \leftarrow [h, \infty)\ \forall s \in S_0$
\Repeat
    \State Update predicted safe spaces $\mathcal{X}_t^-$ and $\mathcal{X}_t^+$
    \State Update set of potential expanders $G_t$ with \eqref{eq:expanders}
    \State Set $\mathcal T_t \leftarrow \mathcal G_t$ if $\mathcal G_t \neq \emptyset$, else $\mathcal X_t^-$
    \State Select target state $\xi \leftarrow \arg\max_{s \in \mathcal{T}_t} w_t(s)$
    \State Go to $\xi$ (via shortest path within $\mathcal{X}_t^-$), receive~$\tilde{m}(s_t)$
    \State Update $\hat{m}(s_t)$ using ($\theta_t, \kappa_t, \nu_t$)
\Until{$\max_{s \in \mathcal{X}_t^-} w_t(s) \le 2\bar{\beta}$}
\Loop
    \State Update optimistic reward with \eqref{eq:reward_update}
    \State Compute optimal value function $J^*(s_t)$ over $\mathcal{X}_t^-$
    \State Execute action $a_t \coloneqq \arg \max _{a} \mathbb E [J^\ast (s_t)]$
\EndLoop
\end{algorithmic}
\end{algorithm}

We split the algorithm into two phases. The first phase is dedicated entirely to the exploration of safety. To efficiently expand the pessimistic safe space $\mathcal{X}_t^-$, the agent identifies a set of \textit{expanders} $G_t$, defined as
\begin{align}
\label{eq:expanders}
G_t = \{s \in \mathcal{X}_t^- \mid \exists a \in \mathcal{A} : f(s, a) \in \mathcal{X}_t^+ \setminus \mathcal{X}_t^-\}.
\end{align}
These are states within $\mathcal{X}_t^-$ from which safety can potentially be extended to currently uncertified states via Assumption~\ref{ass:lipschitz}. We greedily select the state  $\xi \in G_t$ with the largest confidence interval width $w_t(s) = u_t(s) - l_t(s)$. This difference is the measure of uncertainty about $g$ at $s$, so we sample where our knowledge of the safety function is the weakest. If $G_t$ is empty, we default to reducing the highest uncertainty remaining within the currently certified safe space. While traversing to the target $\xi$, we receive new measurements $\tilde{m}(s_t)$ of the reward and safety functions and update the local statistics $(\theta_t, \kappa_t, \nu_t)$. This exploration phase terminates when the maximum uncertainty across the entire pessimistic safe space drops below a user-defined optimality bound $\bar{\beta}$.

In the second phase, we focus on exploiting the reward. Operating exclusively within the certified safe space $\mathcal{X}_t^-$, we define an optimistic reward function for a pair $(s, a)$ based on the upper confidence bound for the reward component $m_0$ at the successor state $s' \coloneqq f(s, a)$ as 
\begin{align}
    \label{eq:reward_update}
    U_t(s, a) = \min \Big( R_{\mathrm{max}}, u_t^r(s') \Big),
\end{align}
where $u^r_t(s) \coloneqq \hat{m}_{0,t}(s) + \beta_t(s)$ is the upper confidence bound on the reward from Lemma~\ref{lem:bounds}. 

We then compute the optimal value function $J^*(s_t)$ using the Bellman equation and execute the action $a_t$ that maximizes its expected value \cite{suttonbarto}. By updating the estimator with reward measurements while executing this optimistic policy, the agent reduces reward uncertainty. In the next section, we show that this converges to a near-optimal policy without breaching the safety threshold.

\subsection{Safety and optimality}
\label{sec:guarantees}


The safety of \textsc{CoLSafe-MDP} rests on the pessimistic safe set $\mathcal{X}_t^-$. By construction, we restrict the agent's transitions to states within $\mathcal{X}_t^-$, which are certified safe through the confidence bounds from Lemma~\ref{lem:bounds}. The following theorem formalizes this.

\begin{restatable}{theorem}{theoremone}
\label{thm:theorem1}
    Under Assumptions~\ref{ass:noise}--\ref{ass:kernel}, and following Algorithm~\ref{alg:algorithm}, $g(s_t) \ge h$ for all $t \ge 0$ with probability at least $1-\delta$.
\end{restatable}

The proof proceeds by induction on the safe set, using Lemma~\ref{lem:bounds} and Assumption~\ref{ass:lipschitz} to certify that each newly added state satisfies $g(s) \ge h$.

Next, we prove that the algorithm converges to a near-optimal policy. This requires two steps. First, we must show that $\beta_t(s)$ can be driven below any target threshold $\bar{\beta} > L\lambda$ with a finite number of measurements. Under Assumption~\ref{ass:lipschitz}, the residual bias $L\lambda$ does not vanish with more data for a fixed bandwidth $\lambda$.

\begin{figure*}[h!]
    \centering
    \includegraphics[width=\linewidth]{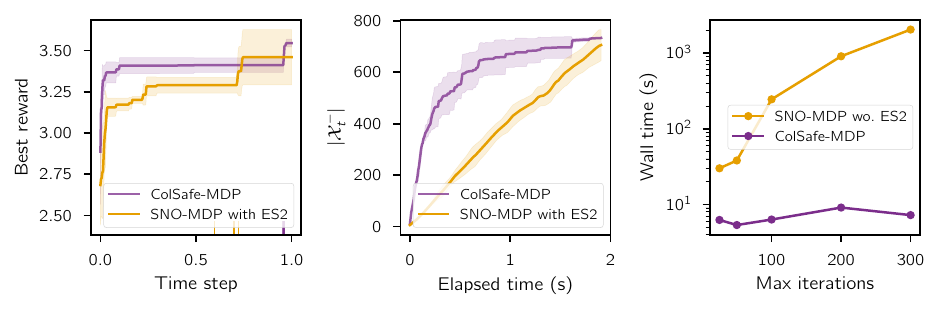}
    \vspace{-2em}
    \caption{Performance comparison of \textsc{CoLSafe-MDP} against the \textsc{SNO-MDP} baseline in the synthetic grid environment. \textit{The leftmost panel shows the best seen reward against a normalized time step count. The dashes at the bottom indicate a transition from Phase 1 to Phase 2. The middle plot shows the size of the certified safe set against wall-clock time. The rightmost plot shows, on a logarithmic scale, the computational time taken by the two algorithms. From these figures, we see that while maintaining the same performance as the baseline, \textsc{CoLSafe-MDP} expands the safe set faster at a substantially lower computational cost.}}
    \label{fig:synthetic_combined}
\end{figure*}

\begin{restatable}{lemma}{lemmatwo}
\label{lem:shrinking}
    Under Assumptions~\ref{ass:noise}--\ref{ass:kernel} and for any $\bar{\beta} > L \lambda$, there exists a finite $N_{\bar{\beta}} \in \mathbb N$, such that an accumulation of $N \ge N_{\bar{\beta}}$ local measurements $s_\tau$ satisfying $\| s - s_\tau \| \le \lambda$ guarantees $\beta_t(s) \le \bar{\beta}$.
\end{restatable}

Intuitively, the noise contribution to $\beta_t(s)$ scales as $\tilde{\mathcal{O}}(1 / \sqrt{N})$ and vanishes with increasing data, while the bias term $L\lambda$ remains constant. Thus, for any margin $\bar{\beta} > L\lambda$, we prove that sufficiently many local observations will drive the total bound below $\bar{\beta}$. We use this result on shrinking uncertainty to show an upper-bound on the exploration phase.

\begin{restatable}{corollary}{corollaryone}
\label{cor:bounded_exploration}
    Under Assumptions~\ref{ass:noise}--\ref{ass:kernel}, a chosen $\bar{\beta} > L \lambda$, and $N_{\bar{\beta}} \in \mathbb N$ as in Lemma~\ref{lem:shrinking}, the safety exploration phase terminates after at most $t^\ast \le D   |\mathcal S| N_{\bar{\beta}}$ environment transitions, where $D$ is the maximum shortest-path distance between any two states in $\mathcal{X}_t^-$ over all $t$. 
\end{restatable}

We show that targeting maximum uncertainty forces it to drop below $2 \bar{\beta}$ in at most $|\mathcal{S}|N_{\bar{\beta}}$ measurements, requiring no more than $D |\mathcal{S}|N_{\bar{\beta}}$ transitions. With the exploration phase now finite, we prove the near-optimality of the resulting policy.

\begin{restatable}{theorem}{theoremtwo}
\label{thm:theorem2}
    For a chosen $\bar{\beta} > L \lambda$, by Corollary~\ref{cor:bounded_exploration}, the exploration phase terminates in $t^\ast \le D | \mathcal S | N_{\bar{\beta}}$ steps. With probability at least $1 - \delta$ and a policy $\pi_t$ executed for all $t \ge t^\ast$,
    \begin{align}
        V^{\pi_t}(s_t) \ge V^\ast (s_t) - \epsilon,
    \end{align}
    where $\epsilon = \frac{2 \bar{\beta}}{1 - \gamma}$ and $V^\ast(s)$ is the value function corresponding to the true optimal policy restricted to the safe space.
\end{restatable}

\begin{remark}
\label{rem:convergence}
The bound $t^\ast \le D  |\mathcal{S}| N_{\bar{\beta}}$ is a worst-case estimate assuming each state requires $N_{\bar{\beta}}$ independent visits. In practice, convergence is faster because each observation also reduces uncertainty at neighbouring states within the kernel's support.
\end{remark}

\section{Evaluation}
\label{sec:evaluation}

We evaluate the performance of \textsc{CoLSafe-MDP} against the baseline \textsc{SNO-MDP} from \cite{wachi2020SafeReinforcementLearning} in two settings. First, we use a synthetic grid-environment to study \textsc{CoLSafe-MDP}'s reward gain, explorativeness, and computational efficiency. Second, we conduct simulations over data from Martian digital imagery to assign safety function values based on the terrain slope.

\subsection{Synthetic environment}
\label{sec:synthetic-environment}

Our synthetic environment, identical to~\cite{wachi2020SafeReinforcementLearning}, is a discrete grid-world MDP framework consistent with Section~\ref{sec:problem}. The state space $\mathcal{S}$ consists of the cells of a two-dimensional rectangular grid, where each cell represents a distinct location the agent can occupy. The action space $\mathcal{A} = \{\text{\textit{stay, up, right, down, left}} \}$ contains five deterministic actions, with transitions that move the agent to the adjacent grid cell in the chosen direction or keep it in place. Each state $s \in \mathcal{S}$ is associated with a scalar reward value $r(s)$ and a scalar safety value $g(s)$, both of which are unknown to the agent and must be estimated online from noisy observations. 

In our experiments, we use a $20 \times 20$ square grid where the reward and safety values for each cell were randomly generated. For \textsc{CoLSafe-MDP}, we chose a box kernel with bandwidth $\lambda = 1$, which corresponds to a compact kernel that weighs two neighbouring states in each direction equally. We set $L = 0.08$, $\sigma = 0.04$, and $\delta = 0.05$. For the baseline algorithm, we chose a Gaussian kernel with length scale $\ell=2$ for the GP that infers reward and $\ell=1$ for the GP that infers safety. We set the confidence interval parameters to $\beta_t = 4.1$\footnote{Here, $\beta_t$ is from \cite{wachi2020SafeReinforcementLearning}, not the bound from Lemma~\ref{lem:bounds}. To successfully reproduce the results of \textsc{SNO-MDP}, we had to choose hyperparameters different from those reported in \cite{wachi2020SafeReinforcementLearning}.}. For both methods, we set the discount factor $\gamma = 0.99$. 

The baseline method has an early-stopping condition (ES2 in~\cite{wachi2020SafeReinforcementLearning}) where the safe set expansion is terminated if the optimal policy for the MDP heads for the inside of $\mathcal X_t^-$. Figure~\ref{fig:synthetic_combined} summarizes this comparison. We observe, in the rightmost panel, that our computational improvement over the baseline is substantial. The wall-clock time stays nearly flat as the iteration count grows from 50 to 300 while the GP-baseline grows with $\mathcal{O}(n^3)$. The middle panel shows that although we require more iterations, \textsc{CoLSafe-MDP} discovers safe states faster than \textsc{SNO-MDP} with ES2. Finally, the leftmost panel shows that this comes at no cost to the learnt policy; both methods reach comparable best-seen reward. The corresponding reward and safety exploration on the environment is visualized in Figure~\ref{fig:heatmaps}. There were no safety violations for both \textsc{CoLSafe-MDP} and the baseline algorithm. 

\subsection{Mars surface exploration}
\label{sec:mars}

\begin{wrapfigure}{R}{0.25\textwidth}
    \centering
    \includegraphics[width=0.24\textwidth]{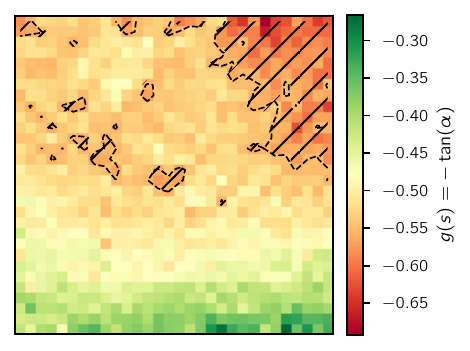}
    \caption{Safety heatmap of the Martian $40 \times 40$-grid considered in our experiments. \textit{Unsafe regions are marked with a dashed line.}}
    \label{fig:mars_heatmap}
\end{wrapfigure}

In this setup, we consider the example of a planetary rover exploring an unknown environment in pursuit of visiting different states to collect samples while avoiding unsafe terrain. Here, we used publicly available data from the High-resolution Imaging Science (HiRISE) camera that was part of the Mars Reconnaissance Orbiter \cite{https://doi.org/10.1029/2005JE002605}. More specifically, we considered a $40 \times 40$ grid from a digital terrain model around a crater-impacted hydrothermal zone located at roughly \SI{31}{\degree}S and \SI{202}{\degree}E\footnote{Link: \href{https://www.uahirise.org/dtm/dtm.php?ID=PSP_010228_1490}{Stratigraphy of Potential Crater Hydrothermal System}, HiRISE, Lunar and Planetary Lab, University of Arizona.}. 

The safety of each state is defined by its slope $\alpha$, computed as $g(s) = -\tan \alpha$. We selected the threshold $h$ based on $\alpha = \SI{29}{\degree}$. This is depicted in Figure~\ref{fig:mars_heatmap}, where the unsafe states comprise one-sixth of the grid. We generated the rewards randomly and set the hyperparameters identically to the experiments in Section~\ref{sec:synthetic-environment}. This setup is in line with~\cite{wachi2020SafeReinforcementLearning}. 

The results from these experiments are presented in Table~\ref{tab:results}. \textsc{CoLSafe-MDP} certifies nearly three times as many safe states against the baseline. Consequently, with more of the terrain available to safely exploit, the learnt policy attains higher mean reward. The computational gap is widest during exploration, where \textsc{CoLSafe-MDP} completes the phase in \SI{0.79}{\second} against \SI{10.73}{\second}, and remains substantial in the exploitation phase. We see that \textsc{CoLSafe-MDP} discovers nearly three times as many safe states whilst achieving higher reward, zero violations, and at faster runtime. Notably, the baseline incurs one safety constraint violation while expanding its safe set, whereas \textsc{CoLSafe-MDP} incurs none in either phase. These results are consistent with the synthetic experiments and further confirm that \textsc{CoLSafe-MDP} scales favourably to the case of a rover navigating an unknown planetary landscape.

\begin{table}[h!]
    \centering
    \caption[Mars terrain experiment results.]{Mars terrain experiment results.\\
      \begin{minipage}{\columnwidth}
      \vspace{0.6em}
      \normalfont\itshape\fontsize{8}{8.2}\selectfont
      For each method, the first row reports the exploration phase and the second the exploitation phase. $|\mathcal{X}^-|$ is the number of certified safe states, $\bar{r}$ the mean reward of the learnt policy, `Viol.' the number of safety constraint violations, and `Time' the wall-clock time in seconds. 
      \vspace{-0.5em}
      \end{minipage}}
    \label{tab:results}
    \footnotesize
    \setlength{\tabcolsep}{5pt}
    \begin{tabular}{lcccc}
        \toprule
        \multicolumn{1}{c}{Method} & $|\mathcal{X}^-|$ & $\bar{r}$ & Viol. & Time (s) \\
        \midrule
        \multirow{2}{*}{\textsc{CoLSafe-MDP}} & -- & \textbf{0.561} & \textbf{0} & \textbf{0.79} \\
                                      & \textbf{1979} & \textbf{0.981} & 0 & \textbf{35.04} \\
        \midrule
        \multirow{2}{*}{\textsc{SNO-MDP-ES2} \cite{wachi2020SafeReinforcementLearning}} & -- & 0.556 & 1 & 10.73 \\
                                      & 674 & 0.567 & 0 & 75.85 \\
        \bottomrule
    \end{tabular}
\end{table}


\begin{figure}[h!]
    \centering
    \includegraphics[width=\linewidth]{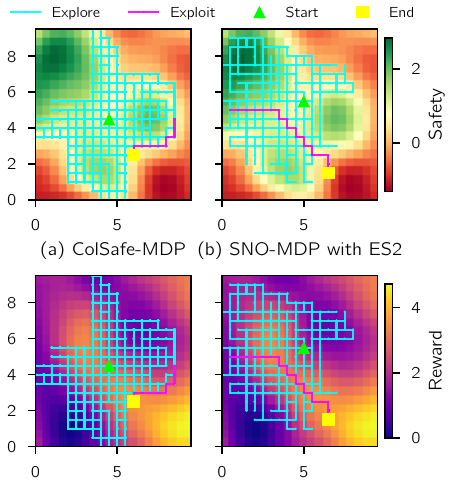}
    \vspace{-2em}
    \caption{Visualizations of the path taken by the agent on the environment, overlaid with reward and safety function values. Top: safety $g(s)$; bottom: reward $r(s)$. \textit{The paths indicate the extent of exploration during the training phase, any potential safety violations, and the final state. The agent visits more safe states with our algorithm, and both methods converge to a similar final position that maximizes reward within safe limits.}}
    \label{fig:heatmaps}
\end{figure}

\section{Conclusion}

In this paper, we present a safe exploration strategy in reinforcement learning that outperforms its counterparts on performance, computational efficiency, and explorativeness under an arguably milder set of assumptions. We adopt the principles behind the \textsc{SNO-MDP} algorithm~\cite{wachi2020SafeReinforcementLearning} and replace the GP estimates with the Nadaraya-Watson estimator. We then update these estimates in $\mathcal O(1)$ time recursively---a significant improvement over the $\mathcal O(n^3)$ computational complexity of GPs. To facilitate the usage of the new estimator in a safe learning setting, we derive safety guarantees on the estimate and near-optimality guarantees on the derived policy that hold with high probability. This overall improvement comes at the cost of sample efficiency. Although outbalanced by the computational improvement, \textsc{CoLSafe-MDP} has a more conservative safe-set expansion strategy that consumes more iterations. 

For future work, it would be relevant to consider extensions that further tighten our bounds, such as by using local Lipschitz constants to reduce the estimator's bias, or by using an adaptive bandwidth parameter $\lambda$. We also inherit the same set of assumptions as baseline methods~\cite{sui2015SafeExplorationOptimization,pmlr-v70-chowdhury17a}. Future extensions could weaken these assumptions. For example, approaches to logically constrained decision making use scenario optimization, which requires fewer assumptions on function smoothness and noise distributions~\cite{tokmak2026safelearningbasedcontrolfunctionbased}. Other extensions would include expanding efficiency and efficacy to higher dimensional settings. 

\section{Proofs}
\label{sec:proofs}

In this section, we prove the safety and optimality guarantees: Lemma~\ref{lem:bounds}, Theorem~\ref{thm:theorem1}, Lemma~\ref{lem:shrinking}, Corollary~\ref{cor:bounded_exploration}, and Theorem~\ref{thm:theorem2}. 

\subsection{Proof of Lemma~\ref{lem:bounds}}

\begin{proof}
We build upon the point-wise Nadaraya-Watson martingale bounds derived in \cite{baumann2025SafetyOptimalityLearningBased}. By Assumption~\ref{ass:noise}, the super-martingale inequality \cite[Cor.~2, Lem.~1]{baumann2025SafetyOptimalityLearningBased} yields that for any fixed state $s \in \mathcal S$, fixed component $i \in \{0, 1\}$, and failure probability $\delta' \in (0, 1)$, the following bound holds uniformly over all $t \ge 1$ with probability at least $1 - \delta'$:
\begin{align}
    \label{eq:pointwise-bound}
    & |\hat{m}_{i,t}(s) - m_i(s)| \\ \nonumber & \le L\lambda + \frac{\sigma}{\kappa_t(s)} \sqrt{2 \ln\left(\frac{1}{\delta'} \sqrt{1 + \nu_t(s)}\right)} \sqrt{1 + \nu_t(s)}.
\end{align}

Now, let $E_{s,i}$ denote the event that the point-wise confidence bound fails at state $s$ for component $i$. Applying Boole's inequality, we have,
\begin{align*}
    \mathbb{P}\left(\bigcup_{s \in \mathcal{S}} \bigcup_{i \in \{0,1\}} E_{s,i}\right) \le \sum_{s \in \mathcal{S}} \sum_{i \in \{0,1\}} \mathbb{P}(E_{s,i}).
\end{align*}
We assign a uniform pointwise failure probability $\delta ' \coloneqq \frac{\delta}{2|\mathcal S|}$ to each pair $(s, i)$. Substituting $\delta'$ into the logarithmic term of \eqref{eq:pointwise-bound}, we obtain
\begin{align*}
    \ln \Big( \frac{1}{\delta '} \sqrt{1 + \nu_t(s)}\Big)  = \ln \Big( \frac{2|S|}{\delta}  \sqrt{1 + \nu_t(s)}\Big).
\end{align*}
Substituting this into \eqref{eq:pointwise-bound} yields the exact definition of $\beta_t(s)$ from Lemma~\ref{lem:bounds}, concluding the proof that $|\hat{m}_i(s) - m_i(s)| \le \beta_t(s)$ holds jointly for all $i \in \{0, 1\}$ and all $s \in \mathcal S$ with probability at least $1 - \delta$. 
\end{proof}

\subsection{Proof of Theorem~\ref{thm:theorem1}}
\begin{proof}
    Let $\mathcal{E}$ denote the event on which $| \hat{m}_i(s) - m_i(s) | \le \beta_t(s)$ holds simultaneously for all $i \in \{0,1\}$, $t \ge 1$, and $s \in \mathcal S$. By Lemma~\ref{lem:bounds}, $\mathbb P(\mathcal E) \ge 1 - \delta$. We show that on $\mathcal E$, safety holds for all $t \ge 0$ by induction.

    By Assumption~\ref{ass:safe-set}, the initial set $S_0$ is safe, with $g(s) \ge h$ for all $s \in \mathcal X_0^-$. For the inductive step, assume that all states $ s \in \mathcal X_{t-1}^-$ are safe. At step $t$, any new certified safe state $s \in \mathcal X_t ^- \setminus \mathcal X_{t-1}^-$ must satisfy the pessimistic expansion condition from~\eqref{eq:pessimistic_safeset_expansion}:
    \begin{align*}
        l_t(s') - L \| s - s' \| \ge h.
    \end{align*}
    On $E$, $g(s') \in Q_\tau(s')$ for every $\tau \ge 1$. Since $C_t(s') = Q_t(s') \cap C_{t-1}(s')$, it follows that $g(s') \in C_t(s')$ for all $t$, and hence $g(s') \ge l_t(s') := \min C_t(s')$. Then, by Assumption~\ref{ass:lipschitz},
    \begin{align*}
        g(s) \ge g(s') - L \| s - s'\| \ge l_t(s') - L \| s - s' \| \ge h.
    \end{align*}
    Thus, on $E$, $\mathcal X_t^-$ contains exclusively safe states for every $t$. Since the agent's transitions are restricted to $\mathcal X_{t}^-$, we conclude that $g(s_t) \ge h$ for all $t \ge 0$ jointly with probability at least $1 - \delta$.
\end{proof}

\subsection{Proof of Lemma~\ref{lem:shrinking}}

\begin{proof}
From Lemma~\ref{lem:bounds} we have bounds as
\begin{align}
    \beta_t(s) = L\lambda + \frac{\sigma}{\kappa_t(s)} \sqrt{2 \ln\left(\frac{2|\mathcal{S}|}{\delta} \sqrt{1 + \nu_t(s)}\right)} \sqrt{1 + \nu_t(s)}. \nonumber
\end{align}
Suppose we have $N$ local measurements within the bandwidth $\lambda$ of state $s \in \mathcal S$. By Assumption~\ref{ass:kernel}, the kernel is bounded such that $0 < \chi_K \le K_\lambda(\cdot) \le 1$. Consequently, $\kappa_t(s) \ge N \chi_K$, and $\nu_t(s) = \sum K_\lambda^2(\cdot) \le N\cdot(1)^2 = N$. 

Substituting these limits, we obtain a conservative upper bound function:
\begin{align*}
    \beta_t(s) \le \underbrace{L\lambda + \frac{\sigma \sqrt{1+N}}{N \chi_K} \sqrt{2 \ln\left(\frac{2|\mathcal{S}|}{\delta} \sqrt{1 + N}\right)}}_{\coloneqq f(N)}.
\end{align*}

As $N \to \infty$, the data-dependent noise term scales as $\tilde{\mathcal{O}}(\sqrt{N \ln N}/N)$. Because this term vanishes asymptotically, the bound tightly converges to the inherent estimator bias: $\lim_{N \to \infty} f(N) = L\lambda$.

By our hypothesis we choose $\bar{\beta} > L \lambda$. By the formal definition of a limit, there must exist a finite integer $N_{\bar{\beta}}$ such that for all sample counts $N \ge N_{\bar{\beta}}$, the estimation noise evaluates to strictly less than or equal to the margin $\bar{\beta} - L\lambda$. Thus, for any $N \ge N_{\bar{\beta}}$, $\beta_t(s) \le f(N) \le \bar{\beta}$, thereby concluding the proof.
\end{proof}

\subsection{Proof of Corollary~\ref{cor:bounded_exploration}}

\begin{proof}
Let $n_t(s)$ be the number of local measurements at state $s$ up to time $t$. By Lemma~\ref{lem:shrinking}, $n_t(s) \ge N_{\bar\beta}$ implies $w_t(s) \le 2\bar\beta$.

At each selection the agent picks $\xi_t = \arg\max_{s \in \mathcal T_t} w_t(s)$, where $\mathcal T_t = \mathcal G_t$ if $\mathcal G_t \neq \emptyset$ and $\mathcal T_t = \mathcal X_t^-$ otherwise. Suppose exploration is still running after $T = |\mathcal S| N_{\bar\beta} + 1$ target selections. Since it has not terminated, $w_t(\xi_t) > 2\bar\beta$ for every $t \le T$, so by the contrapositive of Lemma~\ref{lem:shrinking}, $n_t(\xi_t) < N_{\bar\beta}$ at every selection.

Each selection of a target $\xi_t$ yields at least one measurement at $\xi_t$, so $n_t(s)$ increases by at least one every time $s$ is targeted. Consequently, any state can be targeted at most $N_{\bar\beta}$ times, after which its count satisfies $n_t(s) \ge N_{\bar\beta}$, and it can no longer be selected. With $|\mathcal S|$ states, this bounds the total number of selections by $|\mathcal S| N_{\bar\beta}$, contradicting the assumption that $T = |\mathcal S| N_{\bar\beta} + 1$ selections occurred. Hence exploration terminates within $|\mathcal S| N_{\bar\beta}$ target selections.



If $D$ is the largest shortest-path distance between any two states in $\mathcal X_t^-$, each target selection costs at most $D$ transitions, so exploration terminates after at most $t^\ast \le D |\mathcal S| N_{\bar\beta}$ state transitions.
\end{proof}

\subsection{Proof of Theorem~\ref{thm:theorem2}}

\begin{proof}
%
    %
    When $t \ge t^\ast$ (i.e., the second phase of the algorithm), we operate exclusively over $\mathcal X_t^-$ where the uncertainty $w_t(s) \le 2 \bar{\beta}$. For a deterministic state transition $s' = f(s, a)$, our optimistic reward $U_t(s, a) = u_t^r(s')$ satisfies, by Lemma~\ref{lem:bounds} applied to $m_0$,
    \begin{align}
        U_t(s, a) - r(s') \le 2\beta_t(s') \le 2\bar{\beta}.
    \end{align}

    Let $V_{U_t}^\pi$ denote the value of a policy $\pi$ under the optimistic rewards. By the generalized induced inequality in \cite[Lem.~5]{strehl_analysis_2008}, the value difference for any policy in $\mathcal X_t^-$ is bounded by 
    \begin{align}
        V_{U_t}^{\pi}(s) - V^{\pi}(s) & \le \frac{\max_{s,a} |U_t(s,a) - r(f(s, a))|}{1-\gamma} \nonumber \\ &\le \frac{2\bar{\beta}}{1-\gamma}. \label{eq:value-diff}
    \end{align}
    Let $\pi^*$ be the true optimal safe policy restricted to the safe and reachable states, and $\pi_t$ be the agent's greedy optimistic policy. By Lemma 4 in \cite{wachi2020SafeReinforcementLearning}, the optimistic value of $\pi_t$ upper bounds the true optimal value as
    \begin{align}
        V_{U_t}^{\pi_t}(s) \ge V_{U_t}^{\pi^*}(s) \ge V^{\pi^*}(s) = V^*(s). \label{eq:optimism}
    \end{align}
    Substituting~\eqref{eq:value-diff} in~\eqref{eq:optimism} gives us
    \begin{align}
        V^{\pi_t}(s) \ge V_{U_t}^{\pi_t} (s) - \epsilon \ge V^\ast (s) - \epsilon,
    \end{align}
    where $\epsilon = \frac{2 \bar{\beta}}{1 - \gamma}$.
    This proves that the policy executed for all $t \ge t^\ast$ is $\epsilon$-optimal. 
\end{proof}

\bibliographystyle{IEEEtran}
\bibliography{ref}

\end{document}